\documentclass[11pt]{article}

\usepackage{epsfig}
\usepackage{amstext}
\usepackage{amsthm}
\usepackage{hyperref}

\usepackage{a4wide}
\usepackage{amssymb}
\usepackage{amsfonts}
\usepackage{amsmath}
\usepackage{bm}

\theoremstyle{plain}

\newtheorem{theo}{Theorem}[section]
\newtheorem{lem}[theo]{Lemma}
\newtheorem{prop}[theo]{Proposition}
\newtheorem{cor}[theo]{Corollary}
\newtheorem{defi}[theo]{Definition}
\newtheorem{assumption}[theo]{Assumption}

\newtheorem{rem}[theo]{Remark}

\renewcommand{\cal}{\mathcal}
\newcommand{\J}{{\cal J}}
\newcommand{\F}{{\cal F}}
\newcommand{\G}{{\cal G}}
\newcommand{\etainv}{\eta^{-1}}

\newcommand{\E}{{\mathbb{E}}}
\newcommand{\PP}{{\mathbb{P}}}
\newcommand{\R}{{\mathbb{R}}}
\newcommand{\N}{{\mathbb{N}}}

\newcommand{\lam}{\lambda}

\newcommand{\h}{{\cal H}}

\newcommand{\X}{{\cal X}}

\newcommand{\prf}{\begin{proof}} 
\newcommand{\prfend}{\end{proof}} 
\newcommand{\Y}{{\cal Y}}

\newcommand{\x}{{\bf x}}

\newcommand{\z}{{\bf z}}
\newcommand{\Z}{{\bf Z}}

 \newcommand{\M}{{\cal M}}

\newcommand{\A}{{\cal A}}
\newcommand{\K}{{\cal K}}
\newcommand{\NN}{{\cal N}}
\newcommand{\PPP}{{\cal P}}
\newcommand{\eps}{\varepsilon}

\newcommand{\lamstar}{\lam_{*}}
\newcommand{\lamopt}{\lam_{opt}}
\newcommand{\cals}{{\mathcal S}}
\newcommand{\Lam}{\Lambda}

 \newcommand{\nux}{\nu}

\newcommand{\fo}{f_{\rho}}

\newcommand{\priorle}{\PPP^<}
\newcommand{\priorgr}{\PPP^>}

\newcommand{\tr}[1]{\mathrm{Tr}\left[#1\right]}

\newcommand{\paren}[1]{\left(#1\right)}

\newcommand{\norm}[1]{\left\|#1\right\|}
\newcommand{\snorm}[1]{\left\| \bar B^s\left(#1\right)\right\|}

\newcommand{\abs}[1]{\left\lvert #1 \right\rvert}

\RequirePackage{color}

\newcounter{nbdrafts}
\makeatletter
\newcommand{\checknbdrafts}{
\ifnum \thenbdrafts > 0
\@latex@warning@no@line{**********************************************************************}
\@latex@warning@no@line{* The document contains \thenbdrafts \space draft note(s)}
\@latex@warning@no@line{**********************************************************************}
\fi}
\newcommand{\todo}[1]{\addtocounter{nbdrafts}{1}{\color{red} #1}}
\makeatother

\newcommand{\beq}{\begin{equation}}
\newcommand{\eeq}{\end{equation}}

\numberwithin{equation}{section} 

\title{Adaptivity for Regularized Kernel Methods by Lepskii's Principle} 
\author{Nicole M\"ucke  \quad \href{mailto: nicole.mucke@iit.it}{nicole.mucke@iit.it}}
\date{\today}

\begin{document}

\maketitle

\begin{abstract}
We address the problem of {\it adaptivity} in the framework of reproducing kernel Hilbert space (RKHS) regression. More precisely, we analyze estimators  
arising from a linear regularization scheme $g_\lam$. In practical applications, an important task is to choose the regularization parameter $\lam$ appropriately, 
i.e. based only on the given data and independently on unknown structural assumptions on the regression function. 
An attractive approach avoiding data-splitting is the {\it Lepskii Principle} (LP), also known as the {\it Balancing Principle} is this setting.  
We show that a modified parameter choice based on (LP) is  minimax optimal adaptive, up to $\log\log(n)$. A convenient result is the fact that balancing in $L^2(\nu)-$ norm, 
which is easiest, automatically gives 
optimal balancing in all stronger norms, interpolating between $L^2(\nu)$  and the RKHS. An analogous result is open for other classical approaches to data dependent choices of the 
regularization parameter, e.g. for Hold-Out.
\end{abstract}

\section{Introduction and Motivation}
\label{sec:intro}

We study optimal recovery of the regression function $\fo$ in the framework of reproducing kernel Hilbert space (RKHS) learning. Here we are given random and noisy 
observations of the form
\[  Y_j= \fo(X_j) + \epsilon_j \;, \qquad j=1,...,n    \]
at i.i.d. data points $X_1,...,X_n$, drawn according to some unknown distribution $\nu$ on some input space $\X$, taken as a standard Borel space.  
More precisely, we assume that the observed data $(X_i,Y_i)_{1 \leq i \leq n} \in (\X \times \Y)^n$ are sampled  i.i.d. from an unknown probability measure $\rho$ on 
$\X \times \Y$, with $\E[Y_i|X_i]=\fo(X_i)$\,, so that the distribution of $\eps_i$ may
depend on $X_j$\,, while satisfying $\E[\eps_j|X_j]=0$\,. 
For simplicity, we take the output space $\Y$ as the set of real numbers, but this could be generalized
to any separable Hilbert space, see \cite{optimalratesRLS}. 

In our setting, an estimator $\hat f$ for $\fo$ lies in an hypothesis space $\h \subset L^2(\X, \nu)$, which we choose to be a separable reproducing kernel Hilbert 
space (RKHS), having a measurable positive semi-definite kernel $K: \X \times \X \longrightarrow \R$, satisfying  $\sup_{x \in \X}K(x,x) \leq \kappa^2$.

More precisely, we confine ourselves to estimators $\hat f^\lam$ arising from  the fairly large class of 
{\it spectral regularization methods}, see e.e. \cite{rosasco}, \cite{per}, \cite{DicFosHsu15}, \cite{BlaMuc16}. 
This class of methods contains the well known Tikhonov regularization, Landweber iteration or spectral cut-off.

We recall that while tuning the regularization parameter $\lam$ is essential for spectral regularization to work well, an {\it a priori} choice of the 
regularization parameter is in general not feasible in statistical problems since the choice necessarily depends on unknown structural properties 
(e.g. smoothness of the target function or behavior of the statistical dimension).  
This imposes the need for data-driven {\it a-posteriori} choices of the regularization parameter, which hopefully are optimal in some well defined sense. 
An attractive approach is (some version of) the balancing principle going back to Lepskii's seminal paper \cite{lepskii90} in the context of Gaussian white noise, 
having been elaborated by Lepskii himself in a series of papers  and by other authors,  see e.g. \cite{lepskii92}, \cite{lepskii93}, \cite{golden03}, \cite{birge01}, \cite{mathe06} 
and references therein.

Before we present our somewhat abstract approach, we shall motivate the general idea in a specific example. Denoting by 
$$B: f \in \h \mapsto \int_{\X} f(x) K(x,\cdot) d\nux(x) \in \h $$
the kernel integral operator associated to $K$ and the sampling measure $\nux$, we recall from \cite{BlaMuc16} 
that the optimal regularization parameter (as well as the rate of convergence) is determined by the source condition  
assumption  $||B^{-r} \fo||_{\h} \leq R$ for some constants $r,R>0$ as well as by 
an assumed power decay of the effective dimension 
\[ \NN(\lam) = \tr{B(B + \lam)^{-1}} \leq  C_b\lam^{-1/b}  \]
with  intrinsic dimensionality $b>1$  and by the noise variance $\sigma^2>0$. Error estimates are usually established by 
deriving a {\it bias-variance} decomposition, which looks in this special case as 
\begin{equation}
\label{eq:error_old_decomp}
  \norm{ B^s (\fo-\hat f^{\lam}) }_{\h} \; \lesssim \; 
        C_{s}(\eta)\lam^s\paren{ R\lam^r + \frac{\sigma}{\sqrt n} \lam^{-\frac{b+1}{2b}}  }          \; ,  
\end{equation}        
holding with probability at least $1-\eta$, for any $\eta \in (0,1)$, provided $n$ is big enough. 
Here, the function 
$ \lam \mapsto R\lam^r$ is the leading order of an upper bound for the {\it approximation error} and $\lam \mapsto  \frac{\sigma}{\sqrt n}  \lam^{-\frac{b+1}{2b}} $ is 
the leading order of an 
upper bound for the {\it sample error}. 
We combine all  parameters in a vector 
$(\gamma, \theta)$ with $\gamma=(\sigma, R) \in \Gamma = \R_+ \times \R_+ $ and $\theta=(r, b) \in \Theta = \R_+ \times (1, \infty)$. 
The optimal regularization parameter $\lam_{n, (\gamma, \theta)} $  is chosen by balancing the two leading error terms, more precisely 
by choosing $\lam_{n, (\gamma, \theta)}$  as the unique solution of 
\begin{equation}
\label{eq:example_lam_choice}
R\lam^r =  \sigma \lam^{-\frac{b+1}{2b}} \;,
\end{equation}
leading to the  resulting error estimate 
\[  ||B^s(\fo-\hat f^{\lam_{n, (\gamma, \theta)}} )||_{\h} \; \lesssim \; 2C_{s}(\eta) \lam_{n, (\gamma, \theta)}^{s+r}  \; ,  \]
with probability at least $1-\eta$.  
The associated sequence of estimated solutions 
$(f_{\z}^{\lam_{n,(\gamma,\theta)}})_{n \in \N}$, depending on the regularization parameter $(\lam_{n,(\gamma, \theta)})_{(n,\gamma)\in \N\times \Gamma}$ 
was called weak/ strong minimax optimal over the model family $(\M_{(\gamma, \theta)})_{(\gamma, \theta) \in \Gamma \times \Theta}$ with rate of convergence 
given by $(a_{n,(\gamma,\theta)})_{(n, \gamma)\in \N \times \Gamma}$, {\bf pointwisely   for any fixed $\theta \in \Theta$}. 
\\
\\
However, if the parameter $r$ in the source condition or the intrinsic dimensionality $b>1$ are unknown,  
an {\it a priori} choice of the theoretically best value $\lam_{n, (\gamma, \theta)}$ as in \eqref{eq:example_lam_choice} is impossible. Therefore, it is necessary to use 
some {\it a posteriori} choice of $\lam$, independent of the parameter $\theta=(r, b) \in \Theta$. Our aim is to construct an estimator 
$f^{\hat \lam_{n, \gamma}(\z)}_{\z}$\;, i.e. to find a sequence of regularization parameters $(\hat \lam_{n, \gamma}(\z))_n$, 
without knowledge of $\theta \in \Theta$, but depending on the data $\z$, on $\gamma \in \Gamma$ and on the confidence level, 
such that $f^{\hat \lam_n(\z,\eta, \gamma)}_{\z}$ is {\it (minimax) optimal adaptive} in the sense of Definition \ref{def:weak_adaptive}.

{\bf Contribution: }
More generally, we derive adaptivity in the case where the approximation error is upper bounded by some increasing unknown function $\A(\cdot)$ and where 
\[ \cals(n,\lam) = \sigma \sqrt{\frac{\NN(\lam)}{n\lam}} \]
is an upper bound for the sample error. 
Crucial for our approach is a two-sided estimate of the effective dimension in terms of its empirical approximation. This in particular allows to 
control the spectral structure of the covariance operator through the given input data. 
In summary, our approach achieves: 
\begin{enumerate}
\item A fully data-driven estimator for the whole class of spectral regularization algorithms, which does not use data splitting as e.g. Cross Validation. 
\item Adaptation to unknown smoothness {\bf and} unknown covariance structure.
\item One for all: Balancing in $L^2$ (which is easiest) automatically gives optimal balancing in the stronger ${\mathcal H}$- norm (an analogous result is 
open for other approaches to data dependent choices of the regularization parameter). 
\end{enumerate}

The paper is organized as follows: In Section \ref{sec:empirical_effective_dimension} we provide a two-sided estimate of the effective dimension by its empirical counterpart. 
The main results are presented in Section \ref{sec:balancing}, followed by some specific examples in  Section \ref{sec:application_adaption}. 
A more detailed discussion is given in Section \ref{sec:adapt_discussion}.  
The proofs are collected in the Appendix. 


\section{Empirical Effective Dimension}
\label{sec:empirical_effective_dimension}

The main point of this subsection is a two-sided estimate on the effective dimension by its empirical approximation which is crucial for our entire approach.
We recall the definition of the {\it effective dimension} and introduce its empirical approximation, 
the {\it empirical effective dimension}: For $\lambda \in (0,1]$ we set
\begin{equation}
\label{def:empirical_eff_dim}
  \NN(\lam) = \tr{\;(\bar B+ \lam)^{-1} \bar B\;} \;, \qquad \NN_{\x}(\lam) = \tr{\;(\bar B_{\x}+ \lam)^{-1} \bar B_{\x}\;} \; ,  
\end{equation}  
where we introduce the shorthand notation $\bar B_{x} := \kappa^{-2} B_{\x}$ and similarly $\bar B := \kappa^{-2} B$\,. 
Here $\NN(\lam)$ depends on the marginal $\nu$ (through $B$), but is considered as deterministic, while $\NN_{\x}(\lam)$ is considered as a random variable. 
\\

\begin{prop}
\label{prop:rel_bound}
For any $\eta \in (0,1)$, with probability at least $1-\eta$
\begin{equation}
\label{eq:main}
|\; \NN(\lam) - \NN_{\x}(\lam) \;| \; \leq \; 2\log(4\etainv)\big(1+ \sqrt{\NN_{\x}(\lam)}\big)\left( \frac{2}{\lam  n} + \sqrt{\frac{\NN(\lam)}{n\lam}}  \right)\; ,
\end{equation}
for all $n \in \N^*$ and $\lambda \in (0,1]$.
\end{prop}

\begin{cor}
\label{cor:rel_bound}
  For any $\eta \in (0,1)$, with probability at least $1-\eta$,
  one has 
  \[
  \sqrt{\max(\NN(\lambda),1)} \leq (1 + 4\delta) \sqrt{\max(\NN_{\x}(\lambda),1)} \,,
  \]
  as well as
  \[
  \sqrt{\max(\NN_{\x}(\lambda),1)} \leq (1 + 4 (\sqrt{\delta} \vee \delta^2) ) \sqrt{\max(\NN(\lambda),1)} \,,
  \]
  where $\delta:= 2\log(4\etainv)/\sqrt{n\lambda}$\,. In particular, if $\delta \leq 1$, with probability at least $1-\eta$ one has
  \[  \frac{1}{5} \; \sqrt{\max(\NN(\lambda),1)} \leq \sqrt{\max(\NN_{\x}(\lambda),1)}  \leq 5 \sqrt{\max(\NN(\lambda),1)}  \;.
  \]
\end{cor}




\section{Balancing Principle}
\label{sec:balancing}

In this section, we present the main ideas related to the {\it Balancing Principle} and make the informal presentation from the Introduction more precise. Firstly a definition: 

\begin{defi}
\label{def:weak_adaptive}
Let $\Gamma, \Theta$ be sets and let, for $(\gamma,\theta) \in \Gamma \times \Theta$, $\M_{(\gamma, \theta)}$ be a class of data generating distributions on $\X \times \Y$.
For each $\lam \in (0,1]$ let $(\X \times \Y)^n \ni \z \longmapsto f_{\z}^{\lam} \in \h$
be an algorithm. 
If there is a  sequence $(a_{n,(\gamma,\theta)})_{n \in \N}$
$(\gamma,\theta) \in \Gamma \times \Theta$
and a parameter choice $(\hat \lam_{n, \gamma, \tau}(\z))_{(n, \gamma)\in \N\times \Gamma}$ (not depending on $\theta \in \Theta$) 
such that
\begin{small}
\begin{equation}
\label{eq:weak_adaptive}
   \; \lim_{\tau \to \infty}\; 
   \limsup_{n \to \infty}\;
    \sup_{\rho \in \M_{(\gamma,\theta)}}  \rho^{\otimes n} 
  \left(  \norm{ \bar B^s(f^{\hat \lam_{n, \gamma, \tau}(\z)}_{\z}-\fo)}_{\h} \geq \tau a_{n,(\gamma,\theta)}       \right) = 0   
\end{equation}  
and 
\begin{equation}
\label{lowern}
  \lim_{\tau \to 0}  \liminf _{n \to \infty }\inf_{\hat f} \sup_{\rho \in \M_{(\gamma,\theta)}}  \rho^{\otimes n} 
  \left(  \norm{ \bar B^s( \hat f-\fo)}_{\h} \geq \tau a_{n,(\gamma,\theta)}       \right) > 0,  
\end{equation} 
\end{small}
where the infimum is taken over all estimators $\hat f$,
then the sequence of estimators $(f_{\z}^{\hat \lam_{n, \gamma, \eta}(\z)})_{n \in \N}$ 
is called 
{\it   minimax optimal adaptive over $\Theta$}
and the model family $(\M_{(\gamma,\theta)})_{(\gamma,\theta) \in \Gamma \times \Theta}$, 
with respect to the family of rates $(a_{n,(\gamma,\theta)})_{(n, \gamma)\in \N \times \Gamma}$, 
for the interpolation norm of parameter $s \in [0,\frac{1}{2}]$.
\end{defi}

We remind the reader from \cite{BlaMuc16} that 	
upper estimates typically hold on a class $\M^<_{(\gamma,\theta)}$ 
and lower estimates hold on a possibly different class  $\M^>_{(\gamma,\theta)}$, the model class $\M_{(\gamma,\theta)}$ in the above definition being the 
intersection of both.

To find such an  adaptive estimator, we apply a method which is known in the statistical literature as 
{\it Balancing Principle}. Throughout this section we need

\begin{assumption}
\label{assumption1}
Let $\M$ be a class of models. We consider a discrete set of possible values for the regularization parameter
\[  \Lam_m\; = \; \{ \;  \lam_j \; : \;  0< \lam_0 < \lam_1 < ... < \lam_m     \;\} \; .\]
for some $m \in \N$. 
Let $s \in [0,\frac{1}{2}]$ and $\eta \in (0,1]$.  
We assume to have  the following error decomposition uniformly over the grid $\Lambda_m$: 
\begin{equation}
\label{error_bound}
 \norm{(\bar  B_\x+\lam)^{s}(\fo-f_{\z}^{\lam} ) }_{\h} \; \leq \; C_{s}(m,\eta)\; \lam^s \left(\;  \tilde \A(\lam) +   \tilde \cals(n,\lam) \;\right) \; ,
\end{equation}
where 
\begin{equation}
\label{def:constant}
C_{s}(m,\eta)=C_s\log^2(8|\Lam_m|\etainv)\;,\qquad C_s>0\;,
\end{equation}
with probability at least $1-\eta$, for all data generating distributions from $\M$.
The bounds $\tilde \A(\lam)$  and $  \tilde \cals(n,\lam)$ are given by 
\[ \tilde \cals(n, \lam) = \cals (n ,\lam) +  d_1(n, \lam)\;, \quad \cals (n ,\lam) = \sigma\sqrt{\frac{\tilde \NN(\lam)}{n\lam}}\;, 
\qquad d_1(n, \lam)= \frac{M}{n \lam} \;, \]
with $\tilde \NN(\lam) = \max( \NN(\lam) ,1)$ and 
\[  \tilde \A(\lam)=\A(\lam) + d_2(n)\;, \quad  d_2(n)= \frac{C}{\sqrt{n}}\;, \]
where $\A(\lam)$ is increasing, satisfying $\lim_{\lam \to 0}\A(\lam)=0$ and 
for some constants $C < \infty$, $M<\infty$. We further define $d(n, \lam) := d_1(n,\lam) + d_2(n)$. 
\end{assumption}

We remark that it is actually sufficient to assume \eqref{error_bound} for $s=0$ and $s=\frac{1}{2}$. Interpolation via 
inequality $||B^s f||_{\h} \leq ||\sqrt{B}f||^{2s}_{\h} \; ||f||^{1- 2s}_{\h}$ implies validity of \eqref{error_bound} 
for any $s \in [0,\frac{1}{2}]$.

Note that for any $s \in [0,\frac{1}{2}]$, the map $\lam \mapsto \lam^{s} \cals(n, \lam)$ as well as  $\lam \mapsto \lam^{s} d_1(n, \lam)$ are strictly decreasing in $\lam$. 
Also, if $n$ is sufficiently large and if $\lam$  is sufficiently small, $\tilde \A(\lam) \leq \tilde \cals(n, \lam)$.

We let $$\lamopt(n):= \sup \{\lam:\; \tilde \A(\lam) \leq \tilde \cals(n, \lam) \} \;. $$ 
In this definition we have replaced $\A(\lam)$, $\cals(n, \lam)$ by $\tilde \A(\lam)$ and $\tilde\cals(n, \lam)$, thus including the 
remainder terms $d_1(n, \lam)$ and $d_2(n)$ into our definition of $\lamopt(n)$. It will emerge {\it a-posteriori}, that the definition of $\lamopt(n)$ is not 
affected, since the remainder terms are subleading. But {\it a priori}, this is not known. A correct proof of the crucial oracle inequality in Lemma 
\ref{cor:oracle} below is much easier with this definition of $\lamopt(n)$. It will then finally turn out that the remainder terms are really subleading.

The grid $\Lambda_m$ has to be designed such that the optimal value $ \lamopt(n)$ is contained in $[\lam_0,\lam_m]$. 

The best estimator for $\lamopt(n)$ within $\Lam_m$ belongs to the set 
\[  \J(\Lam_m) \; = \;  \left\{\lam_j \in \Lam_m \; :\; \tilde \A(\lam_j)\leq \tilde \cals(n, \lam_j) \right\}   \]
and is given by
\begin{equation}
\label{lamstardef}
 \lamstar := \max \; \J(\Lam_m)  \; .
\end{equation}
In particular, since we assume that $\J(\Lam_m) \neq \emptyset$ and $ \Lam_m \setminus \J(\Lam_m) \neq \emptyset$, there is some $l \in \N$ 
such that $\lam_l = \lamstar \leq \lamopt(n) \leq \lam_{l+1}$.   
Note also that the choice of the grid $\Lambda_m$ has to depend on $n$.

Before we define the balancing principle estimate of $ \lamopt(n)$, we give some intuition of its possible choice: For any 
$\lam \leq \lamopt(n)$, we have $\tilde \A(\lam)\leq \tilde \cals(n,\lam)$. Moreover, for any $\lam_1 \leq \lam_2$ we have 
\[ \norm{(\bar B_\x + \lam_1)^s f}_{\h} \leq \norm{(\bar B_\x+\lam_2)^sf}_{\h} \;. \]
Finally, since $\lam \mapsto \lam^s \tilde \cals(n,\lam)$ is decreasing, Assumption \ref{assumption1} gives for any two $\lam, \lam' \in \J(\Lam_m)$ satisfying 
$\lam' \leq \lam$, 
with probability at least $1-\eta$ 
\begin{align}  
\label{monest}
\norm{(\bar  B_\x+\lam')^{s}(f_{\z}^{\lam'}-f_{\z}^{\lam} ) }_{\h} &\leq \norm{(\bar  B_\x+\lam')^{s}(\fo-f_{\z}^{\lam'} ) }_{\h} + 
         \norm{(\bar  B_\x+\lam')^{s}(\fo-f_{\z}^{\lam} ) }_{\h}    \nonumber \\
         &\leq \norm{(\bar  B_\x+\lam')^{s}(\fo-f_{\z}^{\lam'} ) }_{\h} + 
         \norm{(\bar  B_\x+\lam)^{s}(\fo-f_{\z}^{\lam} ) }_{\h}     \nonumber \\
  &\leq C_{s}(m,\eta)\; \lam'^s \left(\;  \tilde \A(\lam') +  \tilde \cals(n,\lam') \;\right) \; +  \nonumber \\
      & \qquad \qquad + \; \;C_{s}(m,\eta)\; \lam^s \left(\;   \tilde \A(\lam) +   \tilde \cals(n,\lam) \;\right)  \nonumber  \\
 &\leq   4   C_{s}(m,\eta)\; \lam'^s   \tilde \cals(n,\lam') \;. 
\end{align}
An essential step is to find an empirical approximation of the sample error. In view of Corollary \ref{cor:rel_bound} we define 
\[ \tilde \cals_{\x}(n,\lam)= \cals_{\x}(n,\lam) + d_1(n,\lam)\;, \quad  \cals_{\x}(n,\lam)=\sigma\sqrt{\frac{\tilde \NN_{\x}(\lam)}{n\lam}}\;,  \]
with $\tilde \NN_{\x}(\lam) = \max(\NN_{\x}(\lam), 1)$ and  $\NN_{\x}(\lam)$ the empirical effective dimension given in \eqref{def:empirical_eff_dim}. 
Corollary \ref{cor:rel_bound} implies uniformly in $\lam \in \Lam_m$
\begin{equation}
\label{eq:var_bound}
\frac{1}{5}\tilde \cals_{\x}(n,\lam) \leq \tilde \cals(n,\lam) \leq 5 \tilde \cals_{\x}(n,\lam) \;, 
\end{equation}
with probability at least $1-\eta$, 
provided 
\begin{equation}
\label{eq:lam_0}
 n\lam_0 \geq 2 \;,\qquad   2\log(4|\Lam_m|\etainv) \leq      \sqrt{n\lambda_0}   \;.
\end{equation}
Substituting \eqref{eq:var_bound} into the rhs of the estimate \eqref{monest}  motivates our definition of the balancing principle estimate of $ \lamopt(n)$ as 
follows:

\begin{defi}
\label{def:lepest}
Given $s \in [0,\frac{1}{2}]$, $\eta \in (0,1]$  
and $\z \in \Z^n$, 
we set
\begin{small}
\begin{align*}
    \J^+_{\z}(\Lambda_m) &= \{  \;  \lam  \in \Lambda_m \; :\;  
                     || (\bar  B_{\x}+ \lam')^s(f_{\z}^{\lam} - f_{\z}^{\lam'})   ||_{\h}  
         \;  \leq  \;  20C_{s}(m,\eta/2) \; \lam'^s \; \tilde \cals_{\x}(n, \lam') \; , \\
        &    \qquad  \; \forall \lam' \in \Lambda_m, \; \lam' \leq \lam\;    \} 
\end{align*} 
\end{small}
and define 
\begin{equation}
\label{tildelams} 
 \hat \lam_s(\z) :=\max \; \J^+_{\z}(\Lam_m)  \;. 
\end{equation}     
\end{defi}   
Notice that  $\J^+_{\z}(\Lam_m)$ as well as $\hat \lam_s(\z)$ depend on the confidence level $\eta \in (0,1]$.

For the analysis it will be important that the grid $\Lambda_m$ has a certain regularity. We summarize all requirements needed in

\begin{assumption}(on the grid)
\label{assumption2}
\begin{enumerate}
\item Assume that $\J(\Lam_m) \not = \emptyset$ and $\Lam_m \setminus \J(\Lam_m) \not = \emptyset$. 
\item (Regularity of the grid)
There is some $q>1$ such that the elements in the grid obey $1< \lam_{j+1}/\lam_j \leq q$, $j=0,...,m$. 
\item Choose $\lam_0 = \lam_0(n)$ as the unique solution of $n\lam=\NN(\lam)$. We require that $n$ is sufficiently large, such that 
$\NN(\lam_0(n))\geq 1$ (so that the maximum in the definition of $\tilde \NN(\lam)$ can be dropped).  
We further assume that $n\lam_0\geq 2$.  
\end{enumerate}
\end{assumption}

Note that 
$\lam_0(n) \to 0$ as $n \to \infty$. Then, since $\NN(\lam) \to \infty$
as $ \lam \to 0$, we get that this $\lam_0=\lam_0(n)$ satisfies $\lam_0 n= \NN(\lam_0) \to \infty$. 
Furthermore, a short argument shows that the optimal value  $\lamopt(n)$ indeed satisfies $\lam_0 \leq \lamopt(n)$, if $n$ is big enough. 
Since 
$\A(\lam) \to 0$ as $\lam \to 0$, we get $\tilde \A(\lam_0(n)) \to 0 $ as $n \to \infty.$   Since $\tilde \cals(n,\lam_0(n))=1+ \frac{M}{n\lam_0(n)}$ by definition, it follows 
$\tilde \A(\lam_0(n))  \leq \tilde \cals(n,\lam_0(n))$ for $n$ big enough. From the definition of $\lamopt(n)$ as a supremum, we actually have $\lam_0(n) \leq \lamopt(n)$, for $n$ sufficiently large.

Under the regularity assumption, we find that 
\begin{equation}
\label{omega}
\tilde S(n,\lam_{j}) < q \tilde S(n,\lam_{j+1}) \;, \quad j=0,...,m \;.
\end{equation}
Indeed, while the effective dimension $\lam \to \NN(\lam)$ is decreasing, the related function $\lam \to \lam \NN(\lam)$ is non-decreasing. Hence we find that
\[ q^{-1} \NN(\lam) = (q\lam)^{-1} \lam \NN(\lam) < (q\lam)^{-1} (q\lam) \NN(q\lam) = \NN(q\lam)  \]
and since $q > 1$
\[ q^{-1}  \tilde \NN(\lam) = \max \left(q^{-1}  \NN(\lam), q^{-1} \right) < \max(\NN(q\lam) , 1) = \tilde \NN(q\lam) \;.\]
Therefore
\[  q^{-1} \cals(n, \lam_{j})= \sigma \sqrt{\frac{q^{-1}\tilde \NN(\lam_{j})}{nq\lam_{j}}} < \sigma\sqrt{\frac{\tilde \NN(\lam_{j+1})}{n\lam_{j+1}}} = \cals(n, \lam_{j+1}) \;. \]
One also easily verifies that
\[  d_1(n, \lam_j) = \frac{M}{n\lam_j} \leq   \frac{qM}{n\lam_{j+1}} = qd_1(n, \lam_{j+1})\;, \]
implying \eqref{omega}.

\begin{rem}
\label{rem:geom_grid}
The typical case for Assumption \ref{assumption2} to hold is given when the parameters $\lam_j$ follow a geometric progression, i.e., for some $q>1$ we let 
$\lam_j:=\lam_0q^j$, $j=1,...,m$ and with $\lam_m =1$. In this case we are able to  upper bounding the total number of grid points $|\Lam_m|$ in terms of  $\log(n)$. 
In fact, since $\lam_m=1= \lam_0q^m$, simple calculations lead to
\[ |\Lambda_m| = m+1 = 1-\frac{\log(\lam_0)}{\log(q)} \;.  \]
Recall that the starting point $\lam_0$ is required to obey $\NN(\lam_0)=n\lam_0 \geq 2$ if $n$ is sufficiently large, 
implying $-\log(\lam_0) \leq -\log\paren{\frac{2}{n}} \leq \log(n)$. 
Finally, we obtain for $n$ sufficiently large 
\begin{equation}
\label{est:gridlog}
|\Lambda_m| \leq C_q\log(n) \;,
\end{equation}
with $C_q = \log(q)^{-1} + 1$. 
\end{rem}

We shall need an additional assumption on the effective dimension: 

\begin{assumption}
\label{def:eff_dim_low}
\begin{enumerate}
\item
For some $\gamma_1 \in (0,1]$ and for any $\lam $ sufficiently small
\begin{equation*}
\NN(\lam ) \geq C_1 \lam^{-\gamma_1} \;,
\end{equation*}
for some $C_1>0$. 
\item
For some $\gamma_2 \in (0,1]$ and for any $\lam $ sufficiently small
\begin{equation*}
\NN(\lam ) \leq C_2 \lam^{-\gamma_2} \;,
\end{equation*}
for some $C_2>0$. 
\end{enumerate}
\end{assumption}

Note that such an additional assumption restricts the class of admissible marginals and shrinks the class $\M$ in Assumption \ref{assumption1} to a subclass $\M'$. 
Such a lower and upper bound will hold in all examples which we encounter in Section \ref{sec:application_adaption}.   

We further remark that Assumption \ref{def:eff_dim_low} ensures a precise asymptotic behavior for $\lam_0=n^{-1}\NN(\lam_0)$ of the form 
\begin{equation}
\label{est:lam0}
C_{\gamma_1}  \paren{ \frac{1}{n}}^{ \frac{1}{1+\gamma_1} }    \; \leq \;  \lam_0(n)  \;  \leq \; C_{\gamma_2}      \paren{ \frac{1}{n}}^{ \frac{1}{1+\gamma_2} } \;,
\end{equation}
for some $C_{\gamma_1} >0$, $C_{\gamma_2}>0 $.


\subsubsection{Main Results}



The first result is of preparatory character.

\begin{prop}
\label{maintheorem_lepskii}
Let Assumption $\ref{assumption1}$ be satisfied.  Define $\lamstar$ as in \eqref{lamstardef}. 
Assume $n\lam_0 \geq 2$. Then for any 
$$\eta \geq \eta_n:=\min\paren{1\;, \; 4|\Lam_m|\exp\paren{-\frac{1}{2}\sqrt{\NN(\lam_0(n))}} } \;,$$
uniformly over $\M$, with probability at least $1-\eta$
\[ \norm{ (\bar B_{\x} + \lamstar )^s( f_{\z}^{\hat\lam_s(\z)} - \fo )}_{\h} \; \leq \; 102C_s(m,\eta/2)\lamstar^s  \; \tilde \cals(n,\lamstar)  \;. \]
\end{prop}

We shall need 

\begin{lem}
\label{cor:oracle}
If Assumption $\ref{assumption2}$ holds, then
\begin{equation}
\label{eq:oracle_step}
  \lamstar^s \; \tilde \cals(n,\lamstar)  \leq  q^{1-s}   \min_{\lam \in [\lam_0, \lam_m]} \{\; \lam^s\;(\tilde \A(\lam) + \tilde \cals(n, \lam) )\; \} \; .
\end{equation}  
\end{lem}

We immediately arrive at our first main result of this section:

\begin{theo}
\label{cor:balancing_final}
Let Assumption $\ref{assumption1}$ be satisfied and  
suppose the grid obeys Assumption $\ref{assumption2}$.  
Then for any 
$$\eta \geq \eta_n:=\min\paren{1\;, \; 4|\Lam_m|\exp\paren{-\frac{1}{2}\sqrt{\NN(\lam_0(n))}} } \;,$$
uniformly over $\M$, with probability at least $1-\eta$
\[  \snorm{f_{\z}^{\hat\lam_s(\z)} - \fo}_{\h} \; \leq \; q^{1-s}\; D_{s}(m,\eta) \; 
  \min_{\lam \in [\lam_0, \lam_m]} \{\; \lam^s(  \tilde \A(\lam) + \tilde \cals(n, \lam) )   \;\}  \;,  \]
with 
\[   D_{s}(m,\eta)= C'_{s}  \log^{2(s+1)}(16|\Lam_m|\etainv) \;,  \]
for some $C'_{s}>0$. 
\end{theo}

In particular, choosing a geometric grid and assuming a lower and upper bound on the effective dimension, we obtain:

\begin{cor}
\label{cor:balancing_final2}
Let Assumption $\ref{assumption1}$, Assumption \ref{assumption2} and Assumption \ref{def:eff_dim_low} be satisfied.  
Suppose the grid is given by a geometric sequence $\lam_j=\lam_0q^j$, with $q>1$, $j=1,...,m$ and with $\lam_m=1$. 
Then for any 
$$\eta \geq \eta_n:=4C_q\log(n)\exp\paren{-C_{\gamma_1, \gamma_2}n^{\frac{\gamma_1}{2(1+\gamma_2)}}} \;,$$
uniformly over $\M'$, with probability at least $1-\eta$
\[  \snorm{f_{\z}^{\hat\lam_s(\z)} - \fo}_{\h} \; \leq \; \tilde D_{s,q}(n,\eta) \; 
  \min_{\lam \in [\lam_0, 1]} \{\; \lam^s(  \tilde \A(\lam) + \tilde \cals(n, \lam) )   \;\}  \;,  \]
with 
\[   \tilde D_{s,q}(n,\eta)= C_{s,q} \log^{2(s+1)}(\log(n)) \log^{2(s+1)}(16\etainv) \;,  \]
for some  $C_{\gamma_1, \gamma_2}>0$ and some $C_{s,q}>0$, provided $n$ is sufficiently large. 
\end{cor}

Note that $\eta_n \to 0$ as $n\to \infty$.

\subsubsection{One for All: $L^2$-Balancing is sufficient !}

This section is due to an idea suggested by  P. Math\'e (which itself was inspired by the work  \cite{blahoffreiss16}) which we have worked out in detail. 
 We define the $L^2(\nu)-$ balancing estimate $\hat \lam_{1/2}(\z)$ according to Definition \ref{def:lepest} by explicitely choosing 
$s=\frac{1}{2}$ (in contrast to Theorem \ref{cor:balancing_final}, where we choose $\hat \lam_{s}(\z)$ depending on the norm parameter $s$).  Our main result states that balancing in the $L^2(\nu)-$ norm suffices to automatically give balancing in all other (stronger !) intermediate norms 
$|| \cdot||_s$, for any $s \in [0, \frac{1}{2}]$.

\begin{theo}
\label{theo:oneforall}
Let Assumption $\ref{assumption1}$ and Assumption \ref{assumption2} be satisfied and  
suppose the grid obeys Assumption $\ref{assumption2}$.  
Then for any 
$$\eta \geq \eta_n:=\min\paren{1\; ,\; 4|\Lam_m|\exp\paren{-\frac{1}{2}\sqrt{\NN(\lam_0(n))}} }\;,$$  
uniformly over $\M$, with probability at least $1-\eta$
\[   \norm{\bar B^s( f_{\z}^{\hat \lam_{1/2}(\z)}  - \fo )   }_{\h} \; \leq \;q^{1-s}  \hat D_{s}(m,\eta)\;
        \min_{\lam \in [\lam_0, \lam_m]} \{\;\lam^s( \tilde \A(\lam) + \tilde \cals(n, \lam)) \; \}  \;,  \]
with
\begin{align*}
\hat D_{s}(m, \eta)=   C'_{s}\log^{2(s+1)}(16|\Lam_m|\etainv) \;,
\end{align*} 
for some $C'_{s}>0$.                       
\end{theo}


In particular, choosing a geometric grid and assuming a lower and upper bound on the effective dimension, we obtain:

\begin{cor}
\label{theo:oneforall2}
Let Assumption $\ref{assumption1}$, Assumption \ref{assumption2} and Assumption \ref{def:eff_dim_low} be satisfied.  
Suppose the grid is given by a geometric sequence $\lam_j=\lam_0q^j$, with $q>1$, $j=1,...,m$ and with $\lam_m=1$. 
Then, 
for $n$ sufficiently large and for any 
$$\eta \geq \eta_n:=4C_q\log(n)\exp\paren{-C_{\gamma_1, \gamma_2}n^{\frac{\gamma_1}{2(1+\gamma_2)}}} \;,$$
uniformly over $\M'$, with probability at least $1-\eta$
\[   \norm{\bar B^s( f_{\z}^{\hat \lam_{1/2}(\z)}  - \fo )   }_{\h} \; \leq \;  q^{1-s}\hat D_{s,q}(n,\eta)\;
        \min_{\lam \in [\lam_0, 1]} \{\;\lam^s( \tilde \A(\lam) + \tilde \cals(n, \lam)) \; \}  \;,  \]
with
\begin{align*}
\hat D_{s,q}(n, \eta)=   C_{s,q}\log^{2(s+1)}(\log(n))\log^{2(s+1)}(16\etainv) \;,
\end{align*} 
for some $C_{\gamma_1, \gamma_2}>0$ and some $C_{s,q}>0$.    
\end{cor}

Note that $\eta_n \to 0$ as $n\to \infty$.


\begin{rem}
Still, our choice for $\lam_0$ is only a theoretical value which remains unknown as it depends on the unknown marginal $\nu$ through the effective dimension $\NN(\lam)$. 
Implementation requires a data driven choice. Heuristically, it seems resonable to proceed as follows. 
Let $q>1$ and $\tilde \lam_j = q^{-j}$, $j=0,1,...$ (we are starting from the right and reverse the order). Define the stopping  index 
\[ \hat j_0:= \min\{ \;j \in \N: \; \cals_{\x}(n,\tilde \lam_j) \geq 5 \; \}  \]
and let $\Lambda = \{  \tilde \lam_{\hat j_0} < ... < \tilde \lam_0=1 \}$. 
Here,   $\cals_{\x}(n,\tilde \lam_j) $ depends on the empirical effective dimension $\NN_{\x}(\lam)$, see \eqref{def:empirical_eff_dim},  which by Corollary \ref{cor:rel_bound}  
is close to the 
unknown effective dimension $\NN(\lam)$. Thus we think that the above choice of $\lam_0$ is reasonable for implementing the dependence of $\lam_0$  on the unknown marginal. 
A complete mathematical analysis is in development.
\end{rem}


\section{Specific Examples}
\label{sec:application_adaption}

We proceed by illustrating some specific examples of our method as described in the previous section. 
In view of our Theorem \ref{theo:oneforall} and Corollary \ref{theo:oneforall2}  
it suffices to only consider balancing in $L^2(\nu)$. We always choose a geometric grid as in Remark \ref{rem:geom_grid}, 
satisfying $\lam_m=1$. 

{\bf (1) The regular case}

We consider the setting of \cite{BlaMuc16}, where the eigenvalues of $\bar B$ decay polynomially (with parameter $b>1$),   
the target function $\fo$ satisfies a H\"older-type source condition
\[ \fo \in \Omega_\nu(r, R):= \{\; f \in \h \;: \; f=\bar B_{\nu}^rh \;, \; ||h||_{\h} \leq R \;\}   \]
and the noise satisfies a Bernstein-Assumption 
\begin{equation}
\label{bernstein_first}
\E[\; \abs{Y - \fo(X)}^{m} \; | \; X \;] \leq \frac{1}{2}m! \; \sigma^2 M^{m-2} \quad \nux - {\rm a.s.} \;,
\end{equation}
for any integer $m \geq 2$ and for some $\sigma > 0$ and $M>0$. We combine all structural parameters in a vector $(\gamma, \theta)$, with 
$\gamma = (M, \sigma, R) \in \Gamma = \R_+^3$ and $\theta=(r, b) \in \Theta = (0, \infty) \times (1, \infty)$.
We are interested in adaptivity over $\Theta$. 


It has been shown in \cite{BlaMuc16},  that the corresponding minimax optimal rate is given by 
\begin{equation*}
  a_n=a_{n,\gamma, \theta} = R \lam_{n,\gamma, \theta}^{r+s} = R \paren{\frac{\sigma^2}{R^2n}}^{\frac{b(r+s)}{2br+b+1}}  \;.
\end{equation*}  
We shall now check validity of our Assumption \ref{assumption1}.
In the following, we assume that the data generating distribution belongs to the class $\M=\M_{(\gamma, \theta)}$, defined in \cite{BlaMuc16}. 
Recall that we let $\lam_0(n)$ be determined as the unique solution of  $\NN(\lam) = n\lam$. Then, 
we have uniformly for all data generating distributions from the class $\M$, with probability at least $1-\eta$, for any $\lam \in \Lambda_m$, 
\[  || (\bar B_{\x}+\lam)^s(f_{\z}^{\lam} - \fo) ||_{\h} \leq C_{s}\log^{2}(8|\Lambda_m|\etainv)\;\lam^s\paren{ \; \tilde \A(\lam) +  \tilde \cals(n, \lam) \; } \;,\]
for $n$ sufficiently large, with
\[   \tilde \A(\lam)=R\lam^r + \frac{Rr}{\sqrt n}1_{(1, \infty)}(r)\;, \quad \tilde \cals(n, \lam)= \sigma\sqrt{\frac{\NN(\lam)}{n\lam}}  + \frac{M}{n \lam }  \;,      \]
where $C_{s}$ does not depend on the parameters $(\gamma , \theta ) \in \Gamma \times \Theta $. 
Remember that the optimal choice for the regularization parameter $\lam_n$ is obtained by solving
\[ \A(\lam)= \sigma \sqrt{\frac{\lam^{- 1/b}}{n \lam} } \]
and belongs to the interval $[\lam_0(n) , 1]$. This can be seen 
by the following argument: If $n$ is sufficiently large 
\[  1= \sqrt{ \frac{\NN(\lam_0(n))}{n\lam_0(n)} }   \geq \sqrt{C_{\beta, b}}R\lam_n^r = \sigma \sqrt{\frac{C_{\beta, b}\lam_n^{-\frac{1}{b}}}{n\lam_n}} \geq \sqrt{ \frac{\NN(\lam_n)}{n\lam_n} } \;,\]  
which is equivalent to $\cals(n, \lam_0(n)) \geq \cals(n, \lam_n)$. Since $\lam \mapsto \cals(n, \lam)$ is strictly decreasing we conclude $\lam_n \geq \lam_0(n)$. Here we use the bound 
$\NN(\lam) \leq C_{\beta, b} \lam^{-\frac{1}{b}}$.  

Recall that we also have corresponding lower bound $\NN(\lam) \geq C_{\alpha, b} \lam^{-\frac{1}{b}}$, since $\nu \in {\priorgr}(b, \alpha)$, 
granting Assumption \ref{def:eff_dim_low}. 

We adaptively choose the regularization parameter $\hat \lam_{1/2}(\z)$ according to Definition \ref{def:lepest} by $L^2(\nu)-$ balancing 
(i.e. by choosing $s=\frac{1}{2}$) and independently from the parameters $b>1$, $r>0$.
Corollary \ref{theo:oneforall2} gives for any $s \in [0, \frac{1}{2}]$, if $n$ is sufficiently large,  with probability at least $1-\eta$ (uniformly over 
$\M$)
\begin{equation}
\label{eq:appl1}
 \norm{\bar B^s( f_{\z}^{\hat \lam_{1/2}(\z)}  - \fo )   }_{\h} 
\leq \; C'_{s,q} C_s(\eta)\;\paren{ \; a_{n} +  \lam_n^s d(n,\lam_n) \;} \;,
 \end{equation}        
where
$$ C_s( \eta)= \log^{2(s+1)}(\log(n))\log^{2(s+1)}(16\etainv),$$
provided that $\eta \geq \eta_n= 4C_q\log(n)\exp\paren{-C n^{\frac{1}{2(b+1)}}}$, for some $C>0$, depending on $\alpha, \beta$ and $b$. 
Recall that $\eta_n \to 0$ as $n \to \infty$. 

In \eqref{eq:appl1} we have used that
\begin{align*}
\min_{\lam \in [\lam_0(n), 1]} \{\;\lam^s( \tilde \A(\lam) + \tilde \cals(n, \lam)) \; \} &\leq \lam^s_n( \tilde \A(\lam_n) + \tilde \cals(n, \lam_n)) \\
&= \lam^s_n( \A(\lam_n) + \cals(n, \lam_n) + d(n, \lam_n)) \;.
\end{align*}
Then $\lam_n^{s}\A(\lam_n) \leq  a_n $ and $\lam_n^{s} \cals(n, \lam_n)\leq C_{b}a_n $ give 
equation \eqref{eq:appl1}.

It remains to show that for $n$ sufficiently large, the remainder $\lam_n^s d(n,\lam_n)$ is of lower order than the rate $a_n$. 
One finds that 
\[
\frac{M}{n\lam_n} = o\paren{C_{ b}\sqrt{\frac{1}{n}\lam_n^{-\frac{b+1}{b}}}} \,, \quad \frac{r}{\sqrt n} = o(\lambda_n^r)\;.
\]
Summarizing the above findings gives
\begin{cor}[from Corollary \ref{theo:oneforall2}]
Let $s \in [0, \frac{1}{2}]$. Choose the regularization parameter $\hat \lam_{1/2}(\z)= \hat \lam_{n, \gamma, \eta}(\z)$ according to Definition \ref{def:lepest} 
by choosing $s=\frac{1}{2}$. 
Then, if $n$ is sufficiently large, for any  
$$\eta \geq \eta_n= 4C_q\log(n)\exp\paren{-C n^{\frac{1}{2(b+1)}}}\;,$$
$(r, b) \in \R_+ \times (1, \infty)$, $(M, \sigma, R) \in  \R_+^3$
\begin{small}
\[   \sup_{\rho \in \M} \rho^{\otimes n}\paren{ \norm{\bar B^s( f_{\z}^{\hat \lam_{1/2}(\z)}  - \fo )   }_{\h} \leq C'_{s,q}\log^{2(s+1)}(16\etainv)\; b_{n} } \geq 1-\eta \;, \]
\end{small}
with $ b_n=\log^{2(s+1)}(\log(n))\; a_{n}$. 
\end{cor}

Now defining $\tau =  C'_{s,q} \log^{2(s+1)}(16\etainv)$ gives 
\[  \eta = 16 \exp\paren{-\paren{\frac{\tau}{C'_{s,q}}}^{1/2(s+1)}}\;,  \]
implying \eqref{eq:weak_adaptive}.

Observing that the results in \cite{BlaMuc16} imply validity of the lower bound \eqref{lowern},  this means:
\begin{cor}
In the sense of Definition \ref{def:weak_adaptive} the sequence of estimators $(f_{\z}^{\hat \lam_{1/2}(\z)})_{n \in \N}=(f_{\z}^{\hat \lam_{n, \gamma, \eta}(\z)})_{n \in \N}$ is adaptive over 
$\Theta$ (up to log-term) 
and the model family $(\M_{(\gamma,\theta)})_{(\gamma,\theta) \in \Gamma \times \Theta}$
with respect to the family of rates $(a_{n,(\gamma,\theta)})_{(n, \gamma)\in \N \times \Gamma}$, 
for all interpolation norms of parameter $s \in [0,\frac{1}{2}]$.
\end{cor}

{\bf (2) General Source Condition, polynomial decay of eigenvalues}

Our approach also applies to the case where the smoothness is measured in terms of a {\it general source condition}, generated 
by some index function, that is,
\[ \fo \in \Omega_{\nu}(\A):= \{\; f \in \h: \; f = \A(\bar B_{\nu})h, \; ||h||_{\h}\leq 1  \;\}    \;,\]
where $\A: (0,1] \longrightarrow \R_+$ is a continuous non-decreasing function, satisfying $\lim_{t \to 0}\A(t) = 0$. 
We keep the noise condition \eqref{bernstein_first} and we choose the parameter $\gamma=(M,\sigma) \in \Gamma=\R_+^2,$ 
$\theta=(\A,b) \in \Theta={\cal F} \times (1,\infty)$, where 
${\cal F}$ denotes either  the class of {\it operator monotone} functions or the class of functions decomposing into an operator monotone part and an 
{\it operator Lipschitz} part. For more details, we refer the interested reader to  \cite{per}, \cite{mathe16}. 

We introduce the class of data-generating distributions
\begin{align*}
\M^<_{(\gamma,\theta)} &= \{ \rho(dx,dy)=\rho(dy|x) \nu(dx); \rho(\cdot|\cdot) \in \K(\Omega_{\nu}(\A)), \nu \in \priorle(b, \beta) \} \;,\\
\M^>_{(\gamma,\theta)} &= \{ \rho(dx,dy)=\rho(dy|x) \nu(dx); \rho(\cdot|\cdot) \in \K(\Omega_{\nu}(\A)), \nu \in \priorgr(b, \alpha) \} \;,
\end{align*}
where $\priorle(b, \beta)$ and $\priorgr(b, \alpha)$ are exactly defined as in \cite{BlaMuc16}. 
Then $\M=\M_{(\gamma,\theta)}$ is defined as the intersection.

From \cite{rastogi17} and \cite{mathe16} (in particular Proposition 4.3) one then gets that Assumption \ref{assumption1} is satisfied: 
Uniformly for all data generating distributions from the class $\M$, with probability at least $1-\eta$,
\[  || (\bar B_{\x}+\lam)^s(f_{\z}^{\lam} - \fo) ||_{\h} \leq C_{s}\log^{2}(8|\Lambda_m|\etainv)\;\lam^s\paren{ \; \tilde \A(\lam) +  \tilde \cals(n, \lam)\; } \;,\]
for $n$ sufficiently large, with
\[ \tilde \A(\lam)= \A(\lam)+ \frac{C}{\sqrt n} \;, \quad \tilde \cals(n, \lam)= \sigma\sqrt{\frac{\NN(\lam)}{n\lam}} + \frac{M}{n \lam }        \]
and
\[  d(n,\lam_n) = \frac{C}{\sqrt n} +  \frac{M}{n \lam } \;. \]

Assuming $\NN(\lam) \leq C_{\beta, b} \lam^{-1/b} \;$, which as above is implied by polynomial asymptotics of the eigenvalues of the covariance operator $\bar B$ specified by the exponent $b$, 
 the sequence of estimators $(f^{\lam_{n,\A,b}}_{z})_n$ (defined via some spectral regularization having prescribed qualification) using the parameter choice 
\begin{equation}  
\label{def:psi}
      \lam_n:=\lam_{n,\A,b}:= \psi_{\A,b}^{-1}\paren{\frac{1}{\sqrt n}}\;, \quad  \quad \psi_{\A, b}(t):= \A(t)t^{\frac{1}{2}\left(\frac{1}{b} +1\right)}\;,  
\end{equation}       
is then minimax optimal, in both $\h-$norm ($s=0$) and $L^2(\nu)-$norm ($s=1/2$) (see \cite{rastogi17}, \cite{mathe16}), with rate 
\begin{equation}
\label{def:rate_adaptive2}
 a_n:= a_{n,\A,b}:= \lam_{n, \A,b}^{s}\;\A\left( \lam_{n, \A, b}\right) \;.
\end{equation}   
This holds pointwisely for any $(\A, b) \in \Theta={\cal F} \times (1,\infty)$.  
The crucial observation is that 
equation \eqref{def:rate_adaptive2}  is precisely the result obtained by balancing the leading order terms for sample and approximation error.

Arguments similar to those in the previous example show that $\lam_n \in [\lam_0(n), 1]$. Recall that $\NN(\lam) \leq C_{\beta,b} \lam^{-\frac{1}{b}}$ and that $\A(\lam) \to 0$ as $\lam \to 0$. Thus, if $n$ is big enough
\[  1= \sqrt{ \frac{\NN(\lam_0(n))}{n\lam_0(n)} }   \geq \sqrt{C_{\beta,b}} \;\A(\lam_n) = \sqrt{C_{\beta,b}} \psi(\lam_n)\lam_n^{-\frac{1}{2}(\frac{1}{b} +1)} \geq \sqrt{ \frac{\NN(\lam_n)}{n\lam_n} } \;,\]  
which is equivalent to $\cals(n, \lam_0(n)) \geq \cals(n, \lam_n)$. Since $\lam \mapsto \cals(n, \lam)$ is strictly decreasing, we conclude that $\lam_n \geq \lam_0(n)$.

Recall that we also have corresponding lower bound $\NN(\lam) \geq C_{\alpha, b} \lam^{-\frac{1}{b}}$, since $\nu \in {\priorgr}(b, \alpha)$, 
granting Assumption \ref{def:eff_dim_low}. 

We again adaptively choose the regularization parameter $\hat \lam_{1/2}(\z)$ according to Definition \ref{def:lepest} by $L^2(\nu)-$ balancing 
(i.e. by choosing $s=\frac{1}{2}$) and independently from the parameters $b>1$, $r>0$.
Corollary \ref{theo:oneforall2} gives for any $s \in [0, \frac{1}{2}]$, if $n$ is sufficiently large,  with probability at least $1-\eta$ (uniformly over 
$\M$)
\begin{equation}
\label{eq:appl2}
 \norm{\bar B^s( f_{\z}^{\hat \lam_{1/2}(\z)}  - \fo )   }_{\h} 
\leq \; C'_{s,q} C_s(\eta)\;\paren{ \; a_{n} +  \lam_n^s d(n,\lam_n) \;} \;,
 \end{equation}        
where
$$ C_s( \eta)= \log^{2(s+1)}(\log(n))\log^{2(s+1)}(16\etainv)\;,$$
provided that 
$$\eta \geq \eta_n= 4C_q\log(n)\exp\paren{-C n^{\frac{1}{2(b+1)}}} \;, $$
for some $C>0$, depending on $\alpha, \beta$ and $b$.


One readily verifies also in this case that the remainder term $ d(n,\lam_n)$ is indeed subleading: 
$$ n^{-1/2} =\psi_{\A,b}(\lam_n)=\lam_n^{\frac{1}{2}(1 + \frac{1}{b})} \A(\lam_n) =
o\left(\A(\lam_n) \right),$$
and moreover
\[
\frac{M}{n\lam_n} = o\paren{C_{ b}\sqrt{\frac{1}{n}\lam_n^{-\frac{b+1}{b}}}} \;.
\]

From Theorem 3.12 in \cite{rastogi17} one then obtains the lower bound \eqref{lowern}. 

Thus, we have proved:

\begin{cor}[from Corollary \ref{theo:oneforall2}]
Let $s \in [0, \frac{1}{2}]$.
Choose the regularization parameter $\hat \lam_{1/2}(\z)= \lam_{n, \gamma, \eta}(\z) $ according to Definition \ref{def:lepest} by $L^2(\nu)-$ balancing. 
Then, if $n$ is sufficiently large, for any 
$$\eta \geq \eta_n= 4C_q\log(n)\exp\paren{-C n^{\frac{1}{2(b+1)}}} \;, $$ 
$\A \in {\cal F}$, $b>1$ and $(M, \sigma, R) \in  \R_+^3$ one has 
\begin{small}
\[   \sup_{\rho \in \M_{(\gamma,\theta)}} \rho^{\otimes n}\paren{ \norm{\bar B^s( f_{\z}^{\hat \lam_{1/2}(\z)}  - \fo )   }_{\h} \leq C'_{s,q}\log^{2(s+1)}(16\etainv)\; b_{n} } \geq 1-\eta \;, \]
\end{small}
with 
\[  b_n=\log^{2(s+1)}(\log(n))\; a_{n} \;.\]
This means that in the sense of Definition \ref{def:weak_adaptive} the sequence of estimators $(f_{\z}^{\hat \lam_{1/2}(\z)})_{n \in \N}=(f_{\z}^{\hat \lam_{n, \gamma, \eta}(\z)})_{n \in \N}$ 
is adaptive over $\Theta$ (up to log-term)  
and the model family $(\M_{(\gamma,\theta)})_{(\gamma,\theta) \in \Gamma \times \Theta}$ 
with respect to the family of rates $(a_{n,\gamma,\theta})_{(n, \gamma)\in \N \times \Gamma}$ from \eqref{def:rate_adaptive2}, 
for all interpolation norms of parameter $s \in [0,\frac{1}{2}]$.
\end{cor}

{\bf (3) Beyond the regular case}

Recall the class of models  considered in \cite{BlaMuc16beyond}:   Let $\gamma=(M, \sigma, R) \in \Gamma = \R^3_+,$
$\Theta=\{ (r, \nu^*, \nu_*) \in \R_+ \times (1,\infty)^2; \nu^* \leq \nu_{*} \}$  and set
\begin{equation}
\label{measureclass:beyondn}
 \M^<_{(\gamma,\theta)} \; := \; \{ \; \rho(dx,dy)=\rho(dy|x)\nux(dx)\; : \; 
\rho(\cdot|\cdot)\in {\cal K}(\Omega_{\nux}(r,R)), \; \nux \in  \priorle(\nu^{*}) \;\} \;,
\end{equation}
\begin{equation}
\M^>_{(\gamma,\theta)} \; := \; \{ \; \rho(dx,dy)=\rho(dy|x)\nux(dx)\; : \; 
\rho(\cdot|\cdot)\in {\cal K}(\Omega_{\nux}(r,R)), \; \nux \in  \priorgr(\nu_{*}) \;\} \;,
\end{equation}
and denote by $\M=\M_{(\gamma,\theta)}$ the intersection.

We shall verify validity of our Assumption \ref{assumption1}.
In the following, we assume that the data generating distribution belongs to the class $\M$. 
Then, 
we have uniformly for all data generating distributions from the class $\M$, with probability at least $1-\eta$, for any $\lam \in \Lambda_m$, 
\[  || (\bar B_{\x}+\lam)^s(f_{\z}^{\lam} - \fo) ||_{\h} \leq C_{s, \nu^*}\log^{2}(8|\Lambda_m|\etainv)\;\lam^s\paren{ \; \tilde \A(\lam) +  \tilde \cals(n, \lam) \; } \;,\]
with
\[  \tilde  \A(\lam)=R\lam^r+ \frac{Rr}{\sqrt n}1_{(1, \infty)}(r) \;, \quad \tilde \cals(n, \lam)= \sigma \sqrt{\frac{\lam_n^{2r}}{n\G(\lam)}} + \frac{M}{n\lam} \;.
    \]
As usual, we shall investigate adaptivity on the parameter space $\Theta$.

We upper bound the effective dimension by applying results from \cite{BlaMuc16beyond}, using the counting function $\F(\lam)$ defined in equation $(2.1)$. 
We obtain 
$$\NN (\lam)  \leq C_{\nu^{*}} \F(\lam)\;,$$
for any $\lam$ sufficiently small. 
We now follow the discussion in Example (1) above, with $\A(\lam)$, $\cals(n,\lam)$, $d_1(n)$, $d_2(n,\lam)$ remaining unchanged.
We shall only use the new  
upper bound on $\cals(n,\lam)$ defined by
$$ \cals_{+}(n, \lam) = \sigma  \sqrt{ \frac{\F(\lam)}{n \lam}} = \sigma \sqrt{ \frac{\lam^{2r}}{n \G(\lam)}} \;. $$ 
This gives, equating $R\lam^r = \cals_+(n, \lam)$, for $n$ sufficiently large  
$$\lam_n=\lam_{n,\theta} =    \G^{-1}\left(\frac{\sigma^2}{R^2n} \right)\; .  $$
Also in this case, $\lam_n$ can shown to fall in the interval $[\lam_0(n), 1]$. Indeed, if $n$ is sufficiently large 
\[  1= \sqrt{ \frac{\NN(\lam_0(n))}{n\lam_0(n)} }   \geq \sqrt{C_{\nu^*}}R\lam_n^r = \sqrt{C_{\nu^*}} \sigma \sqrt{  \frac{\F(\lam_n)}{n\lam_n} } \geq \sigma \sqrt{ \frac{\NN(\lam_n)}{n\lam_n} } \;,\]  
which is equivalent to $\cals(n, \lam_0(n)) \geq \cals(n, \lam_n)$. Since $\lam \mapsto \cals(n, \lam)$ is strictly decreasing, we have $\lam_0(n) \leq \lam_n$, provided $n$ is big enough. 

More refined bounds  for the effective dimension follow from \cite{BlaMuc16beyond}. 
We have
\[  C_{\nu_*} \lam^{-\frac{1}{\nu_*}}  \leq   \NN(\lam)  \leq  C_{\nu^*} \lam^{-\frac{1}{\nu^*}}  \]
and Assumption \ref{def:eff_dim_low} is satisfied. 

We adaptively choose the regularization parameter $\hat \lam_{1/2}(\z)$ according to Definition \ref{def:lepest} by $L^2(\nu)-$ balancing,   
i.e. by choosing $s=\frac{1}{2}$. 
Corollary \ref{theo:oneforall2} gives for any $s \in [0, \frac{1}{2}]$, if $n$ is sufficiently large,  with probability at least $1-\eta$ (uniformly over 
$\M$)
\begin{equation}
\label{eq:appl3}
 \norm{\bar B^s( f_{\z}^{\hat \lam_{1/2}(\z)}  - \fo )   }_{\h} 
\leq \; C'_{s,q} C_s(\eta)\;\paren{ \; a_{n} +  \lam_n^s d(n,\lam_n) \;} \;,
 \end{equation}        
where
$$ C_s( \eta)= \log^{2(s+1)}(\log(n))\log^{2(s+1)}(16\etainv),$$
provided that 
$$\eta \geq \eta_n=4C_q\log(n)\exp\paren{-C_{\nu_* , \nu^*} n^{  \frac{\nu^*}{2\nu_*(1+\nu^*)}  }} \;. $$ 
In \eqref{eq:appl3} we have used that $a_n=\lam_n^{r+s}$ and 
\begin{align*}
\min_{\lam \in [\lam_0(n), 1]} \{\;\lam^s( \tilde \A(\lam) + \tilde \cals(n, \lam)) \; \} &\leq \lam^s_n( \tilde \A(\lam_n) + \tilde \cals(n, \lam_n)) \\
&= \lam^s_n( \A(\lam_n) + \cals(n, \lam_n) + d(n, \lam_n)) \;.
\end{align*}

As above, one readily checks that that the subleading term $d(n, \lam_n)$ is really subleading:
\[  n^{-\frac{1}{2}} = o(\lam_n^r)\;,\quad    \frac{M}{n\lam_n} = o\paren{ \sqrt{\frac{\lam_n^{2r}}{n\G(\lam_n)}} }  \;.\]

Summarizing, we have proved

\begin{cor}[from Corollary \ref{theo:oneforall2}]
Let $s \in [0, \frac{1}{2}]$.
Choose the regularization parameter $\hat \lam_{1/2}(\z)= \lam_{n, \gamma, \eta}(\z) $ according to Definition \ref{def:lepest} by choosing $s=\frac{1}{2}$. 
Then, if $n$ is sufficiently large, for any  
$$\eta \geq \eta_n=4C_q\log(n)\exp\paren{-C_{\nu_* , \nu^*} n^{  \frac{\nu^*}{2\nu_*(1+\nu^*)}  }} \;. $$ 
for any $r>0$, $1<\nu^*\leq \nu_*$, $(M, \sigma, R) \in  \R_+^3$,  one has 
\begin{small}
\[   \sup_{\rho \in \M_{(\gamma,\theta)}} \rho^{\otimes n}\paren{ \norm{\bar B^s( f_{\z}^{\hat \lam_{1/2}(\z)}  - \fo )   }_{\h} \leq C'_{s,q}\log^{2(s+1)}(16\etainv)\; b_{n} } \geq 1-\eta \;, \]
\end{small}
with 
\[  b_n=\log^{2(s+1)}(\log(n))\; a_{n} \;.\]
Moreover, in the sense of Definition \ref{def:weak_adaptive} the sequence of estimators $(f_{\z}^{\hat \lam_{1/2}(\z)})_{n \in \N}=(f_{\z}^{\hat \lam_{n, \gamma, \eta}(\z)})_{n \in \N}$ 
is adaptive over $\Theta$ (up to log-term)  
and the model family $(\M_{(\gamma,\theta)})_{(\gamma,\theta) \in \Gamma \times \Theta}$
with respect to the family of rates $(a_{n,\gamma,\theta})_{(n, \gamma)\in \N \times \Gamma}$, 
for all interpolation norms of parameter $s \in [0,\frac{1}{2}]$.
\end{cor}


\section{Discussion}
\label{sec:adapt_discussion}

\begin{enumerate}
\item
We have shown that it suffices to prove adaptivity only in $L^2(\nu)-$norm, which is the weakest of all our interpolating norms indexed by $s \in [0,1/2]$.
Similar results of this type (an estimate in a weak norm suffices to establish the estimate in a stronger norm) have been obtained e.g. 
in \cite{blahoffreiss16} and also in the recent paper of Lepskii, see \cite{Lepski16}, in a much more general context.

\item
We shall briefly discuss where and how the presentation of the balancing principle in our work improves the results in the existing literature
on the subject. The first paper on the balancing principle for kernel methods, \cite{vitoperros}, did not yet introduce {\em fast rates}, i.e. 
rates depending on the intrinsic dimensionality $b$. Within this framework  the results give - in the wording of the authors - {\em an optimal adaptive  
choice of the regularization parameter for the class of spectral regularization methods}. 
In the sense of our Definition \ref{def:weak_adaptive} the obtained estimators are  optimal adaptive on the parameter space $\Theta=\R_+$ 
with respect to minimax optimal  rates, which depend on $r$ but not on $b$ 
(or more general, not on the effective dimension $\NN(\lam)$). Technically, the authors of \cite{vitoperros}  define their optimal adaptive estimator 
as the minimum of 2 estimators, corresponding to 2 different norms, namely, setting 
\begin{small} 
\[  \J^+_{\z}(\Lam_m)\; = \;  \left\{  \lam_i \in \Lam_m\; : \; \norm{ \bar B^s_{\x}(f_{\z}^{\lam_i} - f_{\z}^{\lam_j})   }_{\h}  
         \;  \leq  \;   4 C_s (\eta)\; \lam_j^s \;  \cals(n, \lam_j) \;,  
        \; j = 0,..., i-1 \right\}   \]
\end{small}        
and defining  $\tilde \lam_s(\z) :=\max \; \J^+_{\z}(\Lam_m)$, their final estimator is given by
\begin{equation}
\label{def:old_adaptive}     
 \hat \lam_s(\z) := \min\{  \tilde \lam_s(\z),  \tilde \lam_0(\z) \} \; . 
\end{equation}
We encourage the reader to directly compare this definition with our definition in \eqref{tildelams}.  
Using the minimum of two estimators  in this way can be traced back 
to the use of an additive error estimate of the form 
\begin{equation}
\label{eq:old_approx}
   \left| \norm{ \bar B^sf}_{\h}  - \norm{\bar B^s_{\x}f}_{\h}\right| \; \leq \;  \sqrt 6 \log(4/\eta) \; n^{-\frac{s}{2}}  \norm{f}_{\h} \; ,
\end{equation}
holding for any $f \in \h$, $s \in [0,1/2]$ and $\eta \in (0,1)$, with probability at least $1-\eta$. Here we have slightly generalized the original estimate in \cite{vitoperros} to all values of $s \in [0,1/2]$. 
\\
\\
In the setting of \cite{vitoperros}, where only slow rates are considered, the variance $\cals(n,\lambda)$ is fully known. However, when considering fast rates (polynomial decay of eigenvalues), $\cals(n,\lambda)$ additionally depends on the unknown parameter $b>1$ and we have to replace the variance by its empirical approximation 
$\cals_{\x}(n,\lam)$. This can effectively achieved by our Corollary \ref{cor:rel_bound}, where we provide a two sided bound
\[   \frac{1}{5} \;  \cals_{\x}(n,\lam) \leq \cals(n,\lam) \leq 5 \;\cals_{\x}(n,\lam) \;.   \] 
Our bound (in a slightly weaker form) is also used in \cite{mathe16}\;  for bounding the variance by its empirical approximation.

In the preprint \cite{mathe16}\; the authors independently present the balancing principle for fast rates. More precisely, in the case of H\"older-type source conditions, it covers the range $\Theta_{hs}$ of  parameters $(r,b)$ of {\em high smoothness} where $b>1$
and $r \geq 1/2(1-1/b)$, which excludes the region of {\em low smoothness}. In addition, their results include more general types of source conditions.
This work started independently from our work on the balancing principle. 
A crucial technical difference is that \cite{mathe16} is still based on using 
\eqref{eq:old_approx} in an essential way. 
%
However,
the discussion proceeds essentially along the traditional  lines of \cite{vitoperros}, using the above mentioned additive error estimates. This makes the region of low smoothness, i.e. $r < 1/2(1-1/b)$, much less accessible 
and leads to an estimator obtained by balancing 
only  on the restricted parameter space $\Theta_{hs}$ (with respect to minimax optimal rates of convergence, which, however, are known on the larger parameter space $\Theta= \R_+ \times (0,\infty)$). As before, the final estimator is taken to be a minimum of 2 estimators corresponding to different norms.

Our modified 
definition of the estimator defined by balancing, avoiding the additive error estimate in equation \eqref{eq:old_approx},   
allows in the case of H\"older type source conditions to obtain an optimal adaptive estimator (up to $\log \log (n)$ term) on the 
parameter space $\Theta= \R_+ \times (1,\infty)$.
The final estimator is constructed somewhat more directly. It is not taken as a minimum of 2 separately constructed estimators. 
Furthermore, our discussion in Example (2) shows how the more general results of \cite{mathe16} on source conditions different from H\"older -type 
can naturally be recovered in our approach. 

\item
Finally we want to emphasize that this notion of  optimal adaptivity is {\em not} quite the original approach of Lepskii. The paper \cite{birge01}
contains an approach to the optimal adaptivity problem in the white noise framework which is closer to the original Lepskii approach and thus somewhat stronger than the weak approach described above, where the optimal adaptive estimator depends on the confidence level. It seems to be a wide open question how to adapt this original approach to the framework of kernel methods, i.e.  constructing  an estimator which is {\it optimal adaptive in Lepskii-sense }(independent of the confidence level $\eta$) and satisfies
\begin{equation}
\label{new}
 \sup_{\theta \in \Theta}\;  \sup_{\gamma \in \Gamma}\; \limsup_{n \to \infty}   \; 
  a^{-1}_{n,(\gamma,\theta)} \; R_n(\tilde f^{\lam_{n, \gamma} (\z)}, \gamma ) \;  < \;  \infty \,,  
  \end{equation}
with $R_n$ being the risk
\[  R_n(\tilde f^{\lam_{n, (\gamma,\theta)} (\z)}, \gamma)=  \sup_{\rho \in {\M_{(\gamma, \theta)}}} 
  \E_{\rho^{\otimes n}}\big[ \| \bar B^s(  \fo - \tilde f^{\lam_{n, \gamma} (\z)}   )\|_{\h}^p \big]^{\frac{1}{p}} \;, \quad p>0\;, s\in[0,1/2] \;,    \] 
and $a_{n, (\gamma, \theta)}$ being a minimax optimal rate.


Here
we always want to take $\Theta$ as the maximal parameter space on which one has minimax optimal rates. For slow rates, i.e. $\Theta=\{r > 0 \}$, the supremum over $\Theta$ in equation \eqref{new} exists.  For fast rates, the boundary of the open set $\{ b > 1 \}$ poses problems at $b=1$, since one looses the trace class condition on the covariance operator $\bar B$ (in which case minimax optimality as in this thesis is not even proved). We remark that, trying to only use the effective dimension and parametrizing it by
$$ \NN(\lambda) = O(\lambda^{-\frac{1}{b}}), $$
(thus redefining somewhat the meaning of $b$) possibly changes the nature of the boundary at $b=1$ and might give existence of the sup. We leave this question for future research. Furthermore we remark that a rigorous proof of non-existence of the sup for our (spectral) meaning of $b$ requires a suitable lower bound exploding as $b \downarrow 1$, similar to the example in \cite{lepskii90}.

A similar type of difficulty (related to the non-existence of the sup) has already been systematically investigated in \cite{lepskii90} and \cite{lepskii93}. In such a case Lepskii has introduced the weaker notion of {\em the adaptive minimax order of exactness} and he also discusses additional log terms. Such estimators (which are not optimally adaptive) are called simply {\em adaptive}. This is related to the situation which we encounter in this section. It is known that e.g. for point estimators, additional $\log$ terms are indispensable. Our situation, however, is different and one could expect to prove optimal adaptivity in future research.
\end{enumerate}



\appendix 

\section{Proofs of Section \ref{sec:empirical_effective_dimension}}

By $S^1$ we denote the Banach space of trace class operators with norm  $||A||_1=\tr{|A|}$. 
Furthermore,  $S^2$ denotes the Hilbert space of Hilbert-Schmidt operators with norm $||A||_2=\tr{A^*A}^{1/2}$. By $||A||$ we denote the operator norm. 

\begin{proof}[Proof of Proposition \ref{prop:rel_bound}]
We formulate in detail all preliminary results, although they are in principle well known.  There are always some subtleties related to inequalities in trace norm.
For a proof of the following results we  e.g. refer to \cite{ReedSimonI}, \cite{dimsj99}:
\begin{enumerate}
\item If $A \in S^1$ is non-negative, then $||A||_1 = \tr{A}$.
\item 
$|\tr{A}| \leq ||A||_1$\;.
\item If $A$ is bounded and if $B \in S^1$ is self-adjoint and positive, then $|\tr{AB}|\leq ||A|| \; |\tr{B}|$\;.
\item If $A, B \in S^2$, then $||AB||_1 \leq ||A||_2 \; ||B||_2$\;. 
\item If $A \in S^1$, then $||A||^2_2 = |\tr{A^*A}| = ||A^*A||_1 \leq ||A|| \; ||A||_1$ \;. 
\end{enumerate}
Consider the algebraic equality
\begin{align}
\label{eq:trace}
(\bar B+ \lam)^{-1} \bar B - (\bar B_{\x}+ \lam)^{-1} \bar B_{\x} &= (\bar B+ \lam)^{-1}(\bar B-\bar B_{\x}) +
 (\bar B+ \lam)^{-1}(\bar B-\bar B_{\x})(\bar B_{\x}+ \lam)^{-1} \bar B_{\x} \nonumber \\
&=: N_1(\lam, \x) + N_2(\lam, \x) \; .
\end{align}

Hence, 
\begin{align} 
\label{eq:Re}
|\; \NN(\lam) - \NN_{\x}(\lam) \;| 
 & \leq  |\tr{ N_1(\lam, \x)}|  + |\tr{ N_2(\lam, \x) }| \;.
\end{align}   
We want to estimate the first term in \eqref{eq:Re} by applying the  Bernstein inequality,  
Proposition \ref{concentration2}. 
Setting $\xi(x)=\tr{(\bar B+ \lam)^{-1}\bar B_{x}}$, $x \in \X$,  gives 
\[   \frac{1}{n} \sum_{j=1}^n \xi (x_j) = \tr{ (\bar B+ \lam)^{-1}\bar B_{\x} } \;, \qquad \E[\xi] = \tr{(\bar B+ \lam)^{-1} \bar B} \; ,    \]
and thus
\[  \left| \frac{1}{n} \sum_{j=1}^n \xi (x_j) - \E[\xi]\right | = |\tr{N_1(\lam ,\x)}| \;.  \]
Recall that $\bar B_x$ is positive and $\tr{\bar B_x}=\kappa^{-2}||S_x||^2_{HS} \leq 1$. Using 3. leads to
\[  |\xi(x)| \leq \norm{(\bar B+ \lam)^{-1}} \tr{\bar B_x} \leq \frac{1}{\lam} \quad a.s. \; .  \]
Note that 
\[ |\xi(x)|=|\tr{ (\bar B+ \lam)^{-1}\bar B_{\x} }| = |\tr{  \bar S_x (\bar B+ \lam)^{-1} \bar S^*_x }| = \tr{AA^*}   \]
with $A=\bar S_x (\bar B+ \lam)^{-1/2}$ and by 1.\;, since $AA^*$ is non-negative.  
Furthermore, using $\E [ \bar B_{x}] = \bar B$,
\[ \E[ |\xi|^2] \leq \frac{1}{\lam} \E[|\xi|]  \leq \frac{1}{\lam}  \E\left [\tr{ \bar S_x (\bar B+ \lam)^{-1} \bar S^*_x } \right] =  \frac{1}{\lam} \tr{ \E[ (\bar B+ \lam)^{-1} \bar B_x ]} = \frac{1}{\lam} \NN(\lam) \; .  \]
As a result, with probability at least $1-\frac{\eta}{2}$ 
\begin{equation}
\label{eq:N1}
|\tr{N_1(\lam, \x)}| \leq 2 \log(4\etainv) \left( \frac{2}{\lam  n} + \sqrt{\frac{\NN(\lam)}{n\lam}}   \right) \;.
\end{equation}
Writing $H =(\bar B_{\x}+ \lam)^{-1} \bar B_{\x}$, 
we estimate the second term in \eqref{eq:Re} using 2. and 4. and obtain
\[  |\tr{ N_2(\lam, \x) }| \leq   || N_2(\lam, \x) ||_1  \leq ||N_1(\lam, \x ) ||_2 \; || H ||_2 \;. \]
From Proposition 5.2. in \cite{BlaMuc16}, we have with probability at least $1-\frac{\eta}{2}$, 
\begin{equation*}
\norm{N_1 }_{2}=||(\bar B+ \lam)^{-1}(\bar B-\bar B_{\x})||_2 \; \leq 
2\log(4\etainv) \left( \frac{2}{n \lam} + \sqrt{\frac{\cal N(\lam)}{n\lam }}  \right)\; .
\end{equation*}
Finally, recalling that $||H || \leq 1$ we get from 5.
\[ ||H||_2 \; \leq \; ||H||^{1/2} \; ||H||^{1/2}_1 \; \leq \; 
 \sqrt{\NN_{\x}(\lam)}  \quad a.s. \;\; ,  \]
where we used that $\tr{H}=\tr{AA^*}$, with $A=S_x(\bar B_{\x}+ \lam)^{-1/2}$ and point 1.\;\;. 
Collecting all pieces gives the result. 
\end{proof}

\begin{proof}[Proof of Corollary \ref{cor:rel_bound}]
  Since $\log(4\etainv)\geq 1$, the inequality of Proposition~\ref{prop:rel_bound} implies that
  with probability at least $1-\eta$:
\[
|\; \NN(\lam) - \NN_{\x}(\lam) \;| \; \leq \; \frac{2\log(4\etainv)}{\sqrt{\lambda n}}\big(1+ \sqrt{\NN_{\x}(\lam)}\big)
\left( \frac{2\log(4\etainv)}{\sqrt{\lambda n}} + \sqrt{\NN(\lam)}\right )\;.
\]
Put $A:= \sqrt{\NN(\lam)}$\,, $B := \sqrt{\NN_{\x}(\lam)}$\,, and $\delta:= \frac{2\log(4\etainv)}{\sqrt{\lambda n}}$\,,
then one can rewrite the above as $\abs{A^2 - B^2} \leq \delta (1+B)(\delta + A)$\,.

Consider the case $A\geq B$. Then the above inequality is $A^2 -A \delta(1+B) - (B^2 + \delta^2(1+B)) \leq 0$.
Observe that the larger root $x^+$ of the quadratic equation $x^2 + bx + c$ (for $b,c \leq 0$) is bounded as
\[
x^+ = \frac{-b + \sqrt{b^2 -4c}}{2} \leq |b| + \sqrt{|c|}\,,
\]
while the smaller root $x^-$ is negative. Hence, for $x \geq 0$
\[ ( x-x^+)( x-x^-) \leq 0  \quad \Longrightarrow \quad  x \leq x^+ \leq |b| + \sqrt{|c|}\,. \]
Applying this to the above quadratic inequality (solved in $A\geq 0$), we obtain
\[
A \leq \delta(1+B) + \sqrt{B^2 + \delta^2(1+B)} \leq (1+\delta)B + \delta + \delta + \delta\sqrt{B}
\leq (1+2\delta)(B\vee 1) + 2\delta.
\]
Similarly, if $B \geq A$, the initial inequality becomes $B^2  - B \delta(\delta +A) - (A^2 + \delta(\delta+A)) \leq 0$\;
solving this in $B$ and bounding as above we get
\[
B \leq \delta(\delta+A) + \sqrt{A^2 + \delta(\delta+A)} \leq (1+ \delta)A + \delta^2 + \delta + \sqrt{\delta A}
\leq (1+2(\delta \vee \sqrt{\delta}))(A \vee 1) + 2(\delta^2 \vee \delta)\,.
\]
The rest of the proof follows by observing that $1 \leq B \vee 1$, $1\leq A \vee 1$ and 
\[ 2(\delta \vee \sqrt{\delta}) + 2(\delta^2 \vee \delta)\leq 4(\sqrt{\delta} \vee \delta^2) \;. \]
\end{proof}

\section{Proofs of Section \ref{sec:balancing}}

\begin{lem}
\label{lem:inclusion}
For any $s \in [0, \frac{1}{2}]$ and $\eta \in (0,1]$, 
with probability at least $1-\eta$ we have $\lamstar \leq \hat \lam_s(\z)$, provided $2\log(4|\Lam_m|\etainv)\leq \sqrt{n\lam_0}$ and $n\lam_0\geq 2$. 
\end{lem}

\begin{proof}[Proof of Lemma \ref{lem:inclusion}]
Let $\lam \in \Lam_m$ satisfy $\lam \leq \lamstar$. We consider 
the decomposition  
\begin{equation*}
\norm{ (\bar B_{\x}+ \lam)^s(f_{\z}^{\lam} - f_{\z}^{\lamstar})   }_{\h}  \;  \leq  \; \norm{ (\bar B_{\x}+\lam)^s(f_{\z}^{\lam} - \fo )   }_{\h}  + \norm{ (\bar  B_{\x}+\lam)^s(  f_{\z}^{\lamstar} - \fo  )   }_{\h}  \;  .
\end{equation*} 
From Assumption \ref{assumption1} and since $\lam \leq \lamstar$ we have
\begin{eqnarray*}
\norm{ ( \bar B_{\x}+ \lam)^s(f_{\z}^{\lam} - \fo )   }_{\h} & \leq & C_s(m,\eta)  \;  \lam^s \;(  \tilde \A(\lam) +  \tilde \cals(n,\lam) ) \\
   & \leq &  2C_s(m,\eta) \lam^s  \;\tilde \cals(n,\lam)   \; ,
\end{eqnarray*}
with probability at least $1-\eta$. 
\\
Since $\lam \leq \lamstar$ we have by  Assumption \ref{assumption1} and by      
recalling the definition of $\lamstar$ and recalling that $\lam \mapsto \lam^s\tilde \cals(n,\lam)$ is decreasing
\begin{eqnarray*}
\norm{( \bar B_{\x}+ \lam)^s(  f_{\z}^{\lamstar} - \fo  )   }_{\h} & \leq &  \norm{( \bar B_{\x}+ \lamstar)^s(  f_{\z}^{\lamstar} - \fo  )   }_{\h} \\
& \leq & C_s(m,\eta) \; \lamstar^s \;( \tilde \A(\lamstar) +   \tilde \cals(n,\lamstar) ) \\
   & \leq & 2  C_s(m,\eta)\lamstar^s  \;\tilde\cals(n,\lamstar)   \\
   & \leq & 2  C_s(m,\eta)\lam^s  \;\tilde \cals(n,\lam)   \; ,
\end{eqnarray*}
with probability at least $1-\eta$. As a result, using \ref{eq:var_bound}, if $2\log(4|\Lam_m|\etainv)\leq \sqrt{n\lam_0}$ and $n\lam_0\geq 2$, with probability at least $1-\eta$
\[ \norm{ (\bar B_{\x}+ \lam)^s(f_{\z}^{\lam} - f_{\z}^{\lamstar})   }_{\h}  \;  \leq  \;  20 C_s(m,\eta/2)\; \lam^s \;  
\tilde \cals_{\x}(n,\lam)  \;,  \]
with $C_s(m,\eta/2)=C_s\log^2(16|\Lam_m|\etainv)$. 
Finally, from the definition $\eqref{tildelams}$ of $\hat \lam_s(\z)$ as a maximum, one has $\lamstar \leq \hat \lam_s(\z)$ with probability at least $1-\eta$. 
\\
\\
\end{proof}


\begin{proof}[Proof of Proposition \ref{maintheorem_lepskii}]
Let Assumption $\ref{assumption1}$ be satisfied.  Define $\lamstar$ as in \eqref{lamstardef}. 
is implied by the sufficient condition
We write
\begin{eqnarray*}
\label{Bx}
\norm{(\bar B_{\x} +  \lamstar)^s(f_{\z}^{\hat \lam_s(\z)} - \fo)}_{\h} & \leq & 
    \norm{(\bar B_{\x} +  \lamstar)^s(f_{\z}^{\hat \lam_s(\z)} -   f_{\z}^{\lamstar} ) }_{\h}  + 
                    \norm{(\bar B_{\x} +  \lamstar )^s(f_{\z}^{\lamstar} - \fo)}_{\h} 
\end{eqnarray*}
and bound each term separately.  
By definition \eqref{tildelams} of $\hat \lam_s(\z)$ , by Lemma \ref{lem:inclusion} 
and by \eqref{eq:var_bound}, with probability at least $1-\frac{\eta}{2}$

\begin{eqnarray*}
\label{Bx1}
\norm{(\bar B_{\x} + \lamstar)^s(f_{\z}^{\hat \lam_s(\z)} -   f_{\z}^{\lamstar})}_{\h} 
& \leq & 20  C_s(m,\eta/2) \lamstar^s \tilde S_{\x}(n,\lamstar ) \nonumber \\
   & \leq & 100 C_s(m,\eta/2) \lamstar^s \tilde S(n,\lamstar ) \; .
\end{eqnarray*}
By Assumption \ref{assumption1} and recalling the definition of $\lamstar$ in \eqref{lamstardef} gives for the second term with probability at least 
$1-\frac{\eta}{2}$
\begin{eqnarray*}
\label{Bx2}
\norm{ (\bar B_{\x} + \lamstar)^s(f_{\z}^{\lamstar} - \fo)}_{\h} & \leq & C_s(m,\eta/2)\;\lamstar^s  \; 
            ( \; \tilde \A(\lamstar) + \tilde S(n,\lamstar ) \;)  \nonumber  \\
      & \leq & 2C_s(m,\eta/2)\;\lamstar^s \;\tilde S(n,\lamstar ) \; .
\end{eqnarray*}
The result follows from collecting the previous estimates. 
\\
\\
\end{proof}


\begin{proof}[Proof of Lemma \ref{cor:oracle}]
Let Assumption \ref{assumption2}, point 1. and  2. be satisfied. 
We distinguish between the following cases: 
\\
\\
{\bf Case 1:} $ \lam  \geq  q \lamstar$ 
\\
Since $\lam \to \tilde \A(\lam)$ is increasing and by \eqref{omega}
\begin{align*}
\lam^s \;(\tilde \A(\lam) + \tilde \cals(n,\lam) ) &\geq \lam^s\; \tilde \A(\lam)\geq (q\lamstar)^s\; \tilde \A(q\lamstar) \nonumber \\
&\geq (q\lamstar)^s\; \tilde \cals(n, q\lamstar )
     \geq q^{s-1}\lamstar^s\;\tilde \cals(n,\lamstar)  \;. 
\end{align*}
{\bf Case 2:} $\lam  \leq  q \lamstar$ 
\\
Again, since $\lam\to \lam^s \tilde\cals(n, \lam)$ is decreasing and by \eqref{omega} we have 
\begin{align*}
\lam^s \;(\tilde \A(\lam) + \tilde \cals(n,\lam)) &\geq \lam^s\; \tilde \cals(n,\lam)  \geq (q\lamstar)^s\;\tilde \cals(n, q\lamstar )   \geq q^{s-1}\lamstar^s\; \tilde\cals(n,\lamstar)  \;.
\end{align*}
The result follows.
\\
\\
\end{proof}


\begin{proof}[Proof of Theorem \ref{cor:balancing_final}]
From Proposition \ref{maintheorem_lepskii} 
we have
\begin{align*}
\norm{ \bar B^s(\fo - f_\z^{\hat \lam_s(\z)})}_{\h} & \leq 15 \log^{2s}(4|\Lambda_m|\etainv)
           \norm{(\bar B_\x + \lamstar)^s (\fo - f_\z^{\hat \lam_s(\z)}) }_{\h} \nonumber \\
&\leq  D_s(m , \eta) \;\lamstar^s\; \tilde \cals(n,\lamstar)  \;, 
\end{align*}
with probability at least $1-\eta$, provided 
$$\eta \geq \eta_n:=\min\paren{1\;, \; 4|\Lam_m|\exp\paren{-\frac{1}{2}\sqrt{\NN(\lam_0(n))}} }  $$ 
and  
where $ D_s(m , \eta) = C'_s\log^{2(s+1)}(16|\Lambda_m|\etainv)$. 
The result follows by applying Lemma \ref{cor:oracle}\;.
\\
\\
\end{proof}



\begin{proof}[Proof of Corollary \ref{cor:balancing_final2}]
The proof follows from Theorem \ref{cor:balancing_final}, by applying \eqref{est:gridlog} and by using the lower bound from Assumption \ref{def:eff_dim_low}. 
More precisely, the condition 
$$\eta \geq \eta_n:=\min\paren{1\;, \; 4|\Lam_m|\exp\paren{-\frac{1}{2}\sqrt{\NN(\lam_0(n))}} } $$
is implied by the sufficient condition
$$\eta \geq \eta_n:=\min\paren{1,4C_q\log(n)\exp\paren{-\frac{\sqrt C_1}{2}\lam_0(n)^{-\frac{\gamma_1}{2}}}} \;,$$ 
which itself is implied by 
$$\eta \geq \eta_n:=C_q\log(n)\exp\paren{-C_{\gamma_1, \gamma_2}n^{\frac{\gamma_1}{2(1+\gamma_2)}}} \;,$$ 
by using \eqref{est:lam0}, provided $n$ is sufficiently large and with $C_{\gamma_1, \gamma_2}=\frac{\sqrt{C_1}}{2}C_{\gamma_2}^{-\frac{\gamma_1}{2}}$.

Moreover, using $1\leq \log(16\etainv)$ for any $\eta \in (0,1]$,  
we obtain 
\begin{align*}
q^{1-s}\; D_{s}(m,\eta)&= q^{1-s}\; C'_{s}  \log^{2(s+1)}(16|\Lam_m|\etainv) \\
&\leq q^{1-s}\; C'_{s} \paren{  \log(C_q \log(n))+  \log(16\etainv)}^{2(s+1)} \\
&\leq q^{1-s}\; C'_{s}\paren{\log(C_q \log(n)) + 1}^{2(s+1)}  \log^{2(s+1)}(16\etainv) \;.
\end{align*}
Moreover, if $n$ is sufficiently large, we have 
$$ \log(C_q \log(n)) \leq \log(C_q) + \log(n) \leq (1 + \log(C_q))\log(n)$$ 
and thus 
$$q^{1-s}\; D_{s}(m,\eta) \leq C_{s,q} \log^{2(s+1)}(\log(n)) \log^{2(s+1)}(16\etainv)=: \tilde D_{s,q}(n,\eta) \;,$$ 
with $C_{s,q} = q^{1-s}\; C'_{s}(1+\log(C_q))^{2(s+1)}$. 
\end{proof}


\[\]

\begin{lem}
\label{cor:oneforall1}
Assume 
$n \lam_0\geq 2$. 
With probability at least $1-\eta$
\begin{align*}
||  f_{\z}^{\hat \lam_0(\z)} - f_{\z}^{\hat \lam_{1/2}(\z)}  ||_{\h} \; \leq \;  D( m,\eta) \; \tilde \cals(n, \lamstar)  \;,
\end{align*} 
provided
$$\eta \geq \eta_n:=\min\paren{1\;, \; 4|\Lam_m|\exp\paren{-\frac{1}{2}\sqrt{\NN(\lam_0(n))}} } $$ and 
with $D(m,\eta)=200\max ( C_{1/2}, C_0)\log^2(16|\Lam_m|\etainv)$.
\end{lem}

\begin{proof}[Proof of Lemma \ref{cor:oneforall1}]
Recall the definition of $\lamstar$ in \eqref{lamstardef} and write 
\begin{align}
\label{eq:null_bock}
||  f_{\z}^{\hat \lam_0(\z)} - f_{\z}^{\hat \lam_{1/2}(\z)}  ||_{\h} &\leq 
      ||  f_{\z}^{\hat \lam_0(\z)} - f_{\z}^{\lamstar}  ||_{\h}  + ||  f_{\z}^{\lamstar} - f_{\z}^{\hat \lam_{1/2}(\z)}  ||_{\h} \;.
\end{align}
By definition of $\hat \lam_0(\z)$, Lemma  \ref{lem:inclusion} and applying \eqref{eq:var_bound}  gives with probability at least 
$1-\frac{\eta}{2}$
\begin{align}
\label{eq:null_bock1}
||  f_{\z}^{\hat \lam_0(\z)} - f_{\z}^{\lamstar}  ||_{\h} &\leq 20C_{0}(m,\eta/2)    \tilde \cals_{\x}(n,\lamstar )  \nonumber\\
                 &\leq 100C_{0}(m,\eta /2)\tilde \cals(n, \lamstar)  \;.
\end{align}
Using $||f||_{\h}\leq \lamstar^{-\frac{1}{2}}||(\bar B_{\x}+\lamstar)^{\frac{1}{2}}f||_{\h}$, Lemma \ref{lem:inclusion} and the definition of $\hat \lam_{1/2}(\z)$ yields with probability at least $1-\frac{\eta}{2}$
\begin{align}
\label{eq:null_bock2}
||  f_{\z}^{\lamstar} - f_{\z}^{\hat \lam_{1/2}(\z)}  ||_{\h} &\leq \lamstar^{-\frac{1}{2}} 
                          ||   (\bar B_{\x} + \lamstar)^{\frac{1}{2}} 
                 ( f_{\z}^{\lamstar} - f_{\z}^{\hat \lam_{1/2}(\z)}   )  ||_{\h}  \nonumber\\
                 &\leq 20C_{1/2}(m,\eta/2)  \tilde \cals_{\x}(n,\lamstar )\nonumber \\
                 &\leq 100C_{1/2}(m,\eta/2)  \tilde \cals(n, \lamstar)     \;. 
\end{align}
In the last step we applied  \eqref{eq:var_bound} once more. Combining \eqref{eq:null_bock1} and \eqref{eq:null_bock2} with \eqref{eq:null_bock} gives the result. 
\\
\\
\end{proof}



\begin{proof}[Proof of Theorem \ref{theo:oneforall}]
Assume $n$ is sufficiently large and 
$$\eta \geq \eta_n= \min\paren{1, 4|\Lambda_m| \exp\left(-\frac{1}{2}\sqrt{\NN(\lam_0(n))} \right)  } \;.$$ 
Recall that 
$C_s(m,\eta)=C_s\log^2(8|\Lambda_m|\etainv)$.  
We firstly show the result for the case where $s=0$ and get the final one from interpolation. 
We write 
\begin{align*}
|| f_{\z}^{\hat \lam_{1/2}(\z)}  - \fo    ||_{\h}  &\leq ||f_{\z}^{\hat \lam_{1/2}(\z)}  -  f_{\z}^{\hat \lam_{0}(\z)} ||_{\h} + 
                                                            || f_{\z}^{\hat \lam_{0}(\z)} - \fo||_{\h}
\end{align*}
and bound each term separately. 
From Proposition \ref{maintheorem_lepskii},  with probability at least $1-\frac{\eta}{2}$
\[  || f_{\z}^{\hat \lam_{0}(\z)} - \fo||_{\h} \leq 102 C_0\log^2(16|\Lambda_m|\etainv) \tilde \cals(n, \lamstar)  \;.\] 
Applying Lemma \ref{cor:oneforall1} yields  with probability at least $1-\frac{\eta}{2}$
\[  ||  f_{\z}^{\hat \lam_0(\z)} - f_{\z}^{\hat \lam_{1/2}(\z)}  ||_{\h} \leq   D( m,\eta) \tilde \cals(n, \lamstar)  \;,  \]
with $D(m,\eta)=200\max(C_0, C_{1/2})\log^2(16|\Lambda_m|\etainv)$. Collecting both pieces  leads to
\begin{equation}
\label{eq:step}
  || f_{\z}^{\hat \lam_{1/2}(\z)}  - \fo    ||_{\h} \; \leq \; D'(m, \eta) \tilde \cals(n, \lamstar)  \;,
\end{equation}
with probability at least $1-\eta$, where $  D'(m, \eta) = C\log^2(16|\Lambda_m|\etainv)$, $C= 302\max(C_0, C_{1/2})$.

Using $||\bar B^sf||_{\h} \leq ||\sqrt{\bar B}f||_{\h}^{2s}\; ||f||_{\h}^{1-2s}$ for any $s \in [0, \frac{1}{2}]$, applying 
Proposition  \ref{maintheorem_lepskii} and \eqref{eq:step} gives with probability at least $1-\eta$
\begin{align*}
\norm{ \bar B^s( f_{\z}^{\hat \lam_{1/2}(\z)}  - \fo )   }_{\h}  &\leq  \tilde C^{2s}\paren{\log^{3}(16|\Lambda_m|\etainv)\;\sqrt{\lamstar} \;\tilde\cals(n, \lamstar) }^{2s} \\
  & \qquad \quad  C^{1-2s} \;\paren{ \log^2(16|\Lambda_m|\etainv) \;\tilde \cals(n, \lamstar)  }^{1-2s} \\
&\leq C'_s\log^{2(s+1)}(16|\Lambda_m|\etainv) \lamstar^s \; \tilde \cals(n, \lamstar) \;, 
\end{align*}
for some $C'_{s}>0$. Finally, the result follows by applying Lemma \ref{cor:oracle}.
\end{proof}

\begin{proof}[Proof of Corollary \ref{theo:oneforall2}]
The proof follows by combining Theorem \ref{theo:oneforall} and the argumentation in the proof of Corollary \ref{cor:balancing_final2}. 
\end{proof}

\section{Concentration Inequality}
\label{app:concentration}

\begin{prop}
\label{concentration2}
Let $(Z , {\cal B}, \PP )$ be a probability space and $\xi$ a random variable on $Z$ with values in a real 
separable Hilbert space ${\cal H}$. Assume that there are two positive constants $L$ and $\sigma$ such that for any $m\geq 2$
\begin{equation}
\label{expecm}
\E\big[ \norm{ \xi - \E[\xi]    }_{\h}^m  \big ] \leq  \frac{1}{2}m!\sigma^2L^{m-2}.
\end{equation}
If the sample $z_1,...,z_n$ is drawn i.i.d. from $Z$ according to $\PP$, then, for any $0<\eta<1$, with probability greater than $1-\eta$
\begin{equation}
\label{upbound2}
\Big\|   \frac{1}{n}\sum_{j=1}^n\xi(z_j)-\E[\xi]   \Big\|_{\cal H} \leq 
2 \log(2\eta^{-1}) \left( \frac{L}{n} + \frac{\sigma}{\sqrt n} \right)\; .
\end{equation}
In particular, $(\ref{expecm})$ holds if
\begin{eqnarray*}
\norm{ \xi (z) }_{\h}&\leq & \frac{L}{2} \qquad a.s. \; ,\\
\E\big[ \norm{\xi }^2_{\h} \big]&\leq & \sigma^2.
\end{eqnarray*}
\end{prop}

\begin{proof}
See \cite{optimalrates,CapYao10}, from the original result of \cite{pinelissakha} (Corollary 1)\,.
\end{proof}



\bibliographystyle{abbrv}
\bibliography{bibliography}

\end{document}